\documentclass[conference]{IEEEtran}

\usepackage{amsmath,amsthm,amssymb}
\usepackage{graphicx}
\usepackage{multirow,rotating,array}

\newtheorem{theorem}{Theorem}[]
\newtheorem{lemma}[theorem]{Lemma}
\newtheorem{proposition}[theorem]{Proposition}
\newtheorem{corollary}[theorem]{Corollary}

\begin{document}

\title{On Power-law Kernels, corresponding Reproducing Kernel Hilbert Space and Applications}
\author{\IEEEauthorblockN{Debarghya Ghoshdastidar, Ambedkar Dukkipati}
\IEEEauthorblockA{Department of Computer Science and Automation\\
Indian Institute of Science, Bangalore - 560012.\\
email: \{debarghya.g,ad\}@csa.iisc.ernet.in}}
\maketitle
\begin{abstract}
The role of kernels is central to machine learning.
Motivated by the importance of power-law distributions in statistical modeling,
in this paper, we propose the notion of power-law kernels to 
investigate power-laws in learning problem.
We propose two power-law kernels by
generalizing Gaussian and Laplacian kernels. This generalization
is based on distributions, arising out of maximization of a
generalized information measure known as nonextensive entropy that is
very well studied in statistical mechanics. We prove that the
proposed kernels are positive  definite, and provide some insights
regarding the corresponding  Reproducing Kernel Hilbert Space
(RKHS). We also study practical significance of both kernels in
classification and regression, and present some simulation results.
\end{abstract}

\section{Introduction}
The notion of `power-law' distributions is not recent, and they first
arose in economics in the studies of Pareto~\cite{Pareto_1906_jour_SocEdLib} hundred
years ago.  
Later, power-law behavior was observed in various fields
such as physics, biology, computer science etc.~\cite{Gutenberg_1954_book_Princeton,Barabasi_1999_jour_Science}, 
and hence the phrase ``ubiquitous power-laws''. 
Though the term was first coined for distributions with a 
negative constant exponent, \textit{i.e.}, $f(x) \propto x^{-\alpha}$,
the meaning of the term has expanded in due course of time
to include various fat-tailed distributions, \textit{i.e.},
distributions decaying at a slower rate than Gaussian distribution.
This class is also referred to as generalized Pareto distributions.

On the other hand, though the generalizations of information measures
were proposed in the beginning of the birth of information theory,
only (relatively) recently their connections with power-law
distributions have been established. While maximization of Shannon 
entropy gives rise to exponential distributions, these generalized
measures give power-law distributions. This actually led to a 
dramatic increase in interest in generalized
information measures and their application to statistics. 

Indeed, the starting point of the theory of generalized measures of
information is due to Alfred R\'{e}nyi~\cite{Renyi_1960_jour_MTA3}. 
Another generalization was introduced by Havrda and Charv\'{a}t~\cite{Havrda_1967_jour_Kybernetika}, 
and then studied by Tsallis~\cite{Tsallis_1988_jour_StatPhy} in statistical mechanics that is known as
Tsallis entropy or nonextensive entropy. Tsallis entropy involves a
parameter $q$, and it retrieves Shannon entropy as $q\to1$.
The Shannon-Khinchin axioms of Shannon entropy have been generalized to this
case~\cite{Suyari_2004_jour_ITTrans}, and this entropy functional has been studied
in information theory, statistics and many other fields.
Tsallis entropy has been used to study power-law behavior in different cases
like finance, earthquakes and network traffic~\cite{Sato_2001_jour_JourPhyConf,Abe_2003_jour_PhyRevE,Abe_2005_jour_PhysicaA}.

In kernel based machine learning, 
positive definite kernels
are considered as a measure of similarity between points~\cite{Scholkopf_2002_book_MIT}.
The choice of kernel is critical to the performance of the learning algorithms,
and hence, many
kernels have been studied in literature~\cite{Christianini_2004_book_Cambridge}.
Kernels based on information theoretic quantities are also commonly used in 
text mining and image processing~\cite{He_2003_jour_SPTrans}. However, such 
kernels are defined on probability measures.
Probability kernels based on Tsallis entropy have also been studied in~\cite{Martins_2009_jour_JMLR}.

In this work, we are interested in kernels based on maximum entropy distributions. 
It turns out that Gaussian, Laplacian, Cauchy kernels, which have been extensively
studied in machine learning, have corresponding distributions,
which are maximum entropy distributions.
This motivates us to look into kernels that correspond to maximum Tsallis entropy distributions, 
also termed as Tsallis distributions. These distributions have inherent advantages 
as they are generalizations of exponential distributions, and
they exhibit power-law nature~\cite{Sato_2001_jour_JourPhyConf,Ghoshdastidar_2012_conf_ISIT}.
In fact, the value of $q$ controls the nature of the power-law tails.

In this paper, we propose a new kernel based on $q$-Gaussian distribution,
which is a generalization of Gaussian,
obtained by maximizing Tsallis entropy under certain moment constraints. 
Further, we introduce a generalization of the Laplace distribution following the same lines,
and propose a similar $q$-Laplacian kernel. We give some insights into
reproducing kernel Hilbert spaces (RKHS) of these kernels.
We prove that the proposed kernels
are positive definite over a range of values of $q$. 
We demonstrate the effect of these kernel by applying them to
machine learning tasks: classification and regression by
SVMs. We provide
results indicating that in some cases, the proposed kernels 
perform better than their counterparts (Gaussian and Laplacian
kernels) for certain values of $q$.

\section{Tsallis distributions}
Tsallis entropy can be obtained by
generalizing the information of a single event
in the definition of Shannon entropy as shown in~\cite{Tsallis_1988_jour_StatPhy},
where natural logarithm is replaced with $q$-logarithm defined as 
{$\ln_{q} x = \frac{x^{1-q}-1}{1-q}$}, {$q \in \mathbb{R}$}, {$q > 0$}, {$q\neq1$}.
Tsallis entropy in a continuous case is defined as~\cite{Dukkipati_2007_jour_PhyA}
\begin{equation}
H_{q}(p) = \frac{1-\displaystyle\int_{\mathbb{R}}\big(p(x)\big)^q\mathrm{d}x}{q-1}, \quad{q\in\mathbb{R}}, q>0, q\neq1.
\end{equation}
This function retrieves the differential Shannon entropy functional as {$q\rightarrow1$}.
It is called nonextensive because of its pseudo-additive nature~\cite{Tsallis_1988_jour_StatPhy}.

Kullback's minimum discrimination theorem~\cite{Kullback_1959_book_Wiley} establishes
connections between statistics and information theory.
A special case is Jaynes' maximum entropy principle~\cite{Jaynes_1957_jour_PhyRev}, by which
exponential distributions can be obtained by maximizing Shannon entropy functional,
subject to some moment constraints.
Using the same principle, maximizing Tsallis entropy under the following constraint
\begin{equation}
q\text{-mean } \langle{x}\rangle_q := \frac{\displaystyle\int_{\mathbb{R}}x\big(p(x)\big)^q\mathrm{d}x}
{\displaystyle\int_{\mathbb{R}}\big(p(x)\big)^q\mathrm{d}x} = \mu, 
\label{q_mean}
\end{equation}
results in a distribution known as $q$-exponential distribution~\cite{Tsallis_1998_jour_PhysicaA}, which is of the form
\begin{equation}
\label{Eq1:formula}
p(x) = {\frac{1}{\mu}} \exp_q\left(-\frac{x}{(2-q)\mu}\right),
\end{equation}
where the $q$-exponential, {$\exp_q(z)$}, is expressed as
\begin{equation}
\label{q_exp}
\exp_q(z) = \big(1 + (1-q)z\big)_+^{\frac{1}{1-q}}. 
\end{equation}
The condition {$y_+=\max(y,0)$} in~\eqref{q_exp} is called the Tsallis cut-off condition,
which ensures existence of $q$-exponential.  If 
a constraint based on the second moment,
\begin{equation}
\\ q\text{-variance } \langle{x^2}\rangle_q := \frac{\displaystyle\int_{\mathbb{R}}(x-\mu)^2 \big(p(x)\big)^q\mathrm{d}x}
{\displaystyle\int_{\mathbb{R}}\big(p(x)\big)^q\mathrm{d}x} = \sigma^2,
\label{q_var}
\end{equation}
is considered along with~\eqref{q_mean},  one obtains the $q$-Gaussian distribution~\cite{Prato_1999_jour_PhyRevE} defined as
\begin{equation}
\label{Gq1:formula}
p(x) = {\frac{\Lambda_{q}}{\sigma\sqrt{3-q}}} \exp_q\left(-\frac{(x-\mu)^2}{(3-q)\sigma^2}\right),
\end{equation}
where {$\Lambda_q$} is the normalizing constant~\cite{Prato_1999_jour_PhyRevE}.
However, instead of~\eqref{q_mean}, if the constraint 
\begin{equation}
\langle{|x|}\rangle_q := \frac{\displaystyle\int_{\mathbb{R}}|x|\big(p(x)\big)^q\mathrm{d}x}
{\displaystyle\int_{\mathbb{R}}\big(p(x)\big)^q\mathrm{d}x} = \beta, 
\end{equation}
is considered, then maximization of Tsallis entropy with only this constraint leads to a
$q$-variant of the doubly exponential or Laplace distribution centered at zero. A translated version of the distribution
can be written as
\begin{equation}
\label{Lq1:formula}
p(x) = {\frac{1}{2\beta}} \exp_q\left(-\frac{|x-\mu|}{(2-q)\beta}\right).
\end{equation}
As {$q\to1$}, we retrieve the
exponential, Gaussian and Laplace distributions as special cases
of~\eqref{Eq1:formula}, \eqref{Gq1:formula} and~\eqref{Lq1:formula},
respectively. The above distributions can be extended to a multi-dimensional
setting in a way similar to Gaussian and Laplacian distributions,
by incorporating 2-norm and 1-norm in~\eqref{Gq1:formula} and~\eqref{Lq1:formula}, respectively.

\section{Proposed Kernels}
\label{kernel}

Based on the above discussion, 
we define the $q$-Gaussian kernel {$G_q:\mathcal{X} \times \mathcal{X} \mapsto \mathbb{R}$},
for a given $q\in\mathbb{R}$,  as
\begin{equation}
G_q(x,y) = \exp_q \left(-\frac{\Vert{x-y}\Vert_2^2}{(3-q)\sigma^2}\right) 
\text{ for all } x,y \in \mathcal{X},
\end{equation}
where {$\mathcal{X}\subset\mathbb{R}^N$} is the input space, and {$q,\sigma \in \mathbb{R}$} are 
two parameters controlling the behavior of the kernel, satisfying
the conditions {$q \neq 1$}, {$q \neq 3$} and {$\sigma \neq 0$}. 
For {$1<q<3$}, the term inside the bracket is non-negative and hence, the kernel is of the form
\begin{equation}
\label{qG_ker}
G_q(x,y) = \left(1 + \frac{(q-1)}{(3-q)\sigma^2} \Vert{x-y}\Vert_2^2 \right)^{\frac{1}{1-q}},
\end{equation}
where {$\Vert.\Vert_2$} is the Euclidean norm.
On similar lines, we use~\eqref{Lq1:formula} 
to define the $q$-Laplacian kernel {$L_q:\mathcal{X} \times \mathcal{X} \mapsto \mathbb{R}$}
\begin{equation}
L_q(x,y) = \exp_q \left(-\frac{\Vert{x-y}\Vert_1}{(2-q)\beta}\right) 
\text{ for all } x,y \in \mathcal{X},
\end{equation}
where {$\Vert.\Vert_1$} is the $1$-norm, and {$q,\beta \in \mathbb{R}$} satisfy the conditions {$q \neq 1$}, {$q \neq 2$} and {$\beta > 0$}. 
As before, for {$1<q<2$}, the kernel can be written as
\begin{equation}
\label{qL_ker}
L_q(x,y) = \left(1 + \frac{(q-1)}{(2-q)\beta} \Vert{x-y}\Vert_1 \right)^{\frac{1}{1-q}}.
\end{equation}

\begin{figure}[b]
\centering
\includegraphics[width=0.45\textwidth]{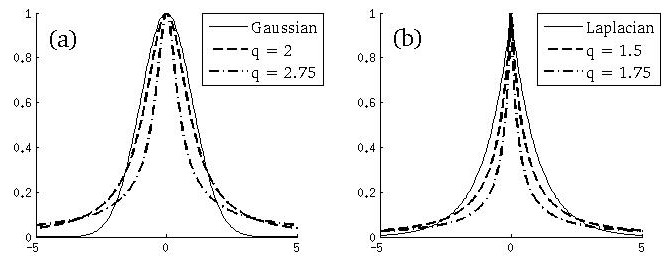}
\caption{Example plots for (a) $q$-Gaussian and (b) $q$-Laplacian kernels with $\sigma=\beta=1$.}
\label{example_plot}
\end{figure}

Due to the power-law tail of the Tsallis distributions for $q>1$, 
in case of the above kernels, similarity
decreases at a slower rate than the Gaussian and Laplacian kernels with increasing distance.
The rate of decrease in similarity is controlled by the parameter $q$, and leads to better performance
in some machine learning tasks, as shown later. 
Figure~\ref{example_plot} shows how the similarity decays 
for both $q$-Gaussian and $q$-Laplacian kernels in the one-dimensional case.
It can be seen that as $q$ increases, the initial decay becomes more rapid,
while towards the tails, the decay becomes slower.
 
We now show that for certain values of $q$, the proposed kernels satisfy the property of positive
definiteness, which is essential for them to be useful in learning theory.
Berg et al.~\cite{Berg_1984_book_Springer} 
have shown that for any symmetric kernel function $K:\mathcal{X}\times\mathcal{X}\mapsto\mathbb{R}$, there
exists a mapping {$\Phi:\mathcal{X}\mapsto\mathcal{H}$}, 
{$\mathcal{H}$} being a higher dimensional space,
such that
{$K(x,y) = \Phi(x)^T\Phi(y)$},	for all {$x,y\in\mathcal{X}$}
if and only if
$K$ is positive definite (p.d.), \textit{i.e.},
given any set of points {$\{x_1, x_2, \ldots, x_n\}\subset\mathcal{X}$},
the {$n\times n$} matrix {$\mathbb{K}$}, such that {$\mathbb{K}_{ij} = K(x_i,x_j)$},
is positive semi-definite. 

We first state some of the results presented in~\cite{Berg_1984_book_Springer}, which are required to prove 
positive definiteness of the proposed kernel.

\begin{lemma}
\label{PD_result1}
For a p.d. kernel {$\varphi:\mathcal{X} \times \mathcal{X} \mapsto \mathbb{R}$}, {$\varphi \geqslant 0$},
the following conditions are equivalent:
\begin{enumerate}
\item
{$-\log \varphi$} is negative definite (n.d.), and
\item
{$\varphi^t$} is p.d. for all {$t > 0$}.
\end{enumerate}
\end{lemma}

\begin{lemma}
\label{PD_result4}
Let {$\varphi:\mathcal{X} \times \mathcal{X} \mapsto \mathbb{R}$} be a n.d. kernel, which is 
strictly positive, then {$\frac{1}{\varphi}$} is p.d.
\end{lemma}

We state the following proposition, which is a general result providing a method to generate p.d. power-law kernels,
given that their exponential counterpart is p.d.
\begin{proposition}
\label{power_pd}
Given a p.d. kernel {$\varphi:\mathcal{X} \times \mathcal{X} \mapsto \mathbb{R}$} of the form
{$\varphi(x,y) = \exp\big(-f(x,y)\big)$}, where {$f(x,y)\geqslant0$} for all {$x,y\in\mathcal{X}$},
the kernel {$\phi:\mathcal{X} \times \mathcal{X} \mapsto \mathbb{R}$} given by
\begin{equation}
\phi(x,y) = \big(1+cf(x,y)\big)^{k}, \quad \text{for all } x,y\in\mathcal{X},
\end{equation}
is p.d., provided the constants {$c$} and {$k$} satisfy the conditions
{$c>0$} and {$k<0$}.
\end{proposition}

\begin{proof}
Since, {$\varphi$} is p.d., it follows from Lemma~\ref{PD_result1} that the kernel
{$f = -\log\varphi$} is n.d.
Thus, for any $c>0$, {$(1 + cf)$} is n.d., and as {$f\geqslant0$},
we can say {$(1 + cf)$} is strictly positive. So, application
of Lemma~\ref{PD_result4} leads to the fact that
{$\frac{1}{(1+cf)}$} is p.d.
Finally, using Lemma~\ref{PD_result1}, we can claim {$(1+cf)^{k}$} is p.d.
for all {$c>0$} and {$k<0$}.
\end{proof}

From Proposition~\ref{power_pd} and positive definiteness of Gaussian and Laplacian kernels, 
we can show that the proposed $q$-Gaussian and $q$-Laplacian kernels are p.d. for certain ranges of $q$.
However, strikingly, it turns out that over this range, the kernels exhibit power-law behavior.

\begin{corollary}
\label{qG_pd}
For {$1<q<3$}, the $q$-Gaussian kernel, as defined in~\eqref{qG_ker}, is positive definite.
\end{corollary}

\begin{corollary}
\label{qL_pd}
For {$1<q<2$}, the $q$-Laplacian kernel, as defined in~\eqref{qL_ker}, is positive definite for all {$\beta>0$}.
\end{corollary}

Now, we show that some of the popular kernels can be obtained as special cases of the proposed kernels.
The Gaussian kernel is defined as
\begin{equation}
\label{psi1}
\psi_1(x,y) = \exp\left(- \frac{\Vert{x-y}\Vert_2^2}{2\sigma^2} \right), 
\end{equation}
where {$\sigma\in\mathbb{R}$}, {$\sigma>0$}. We 
can retrieve the Gaussian kernel~\eqref{psi1} when {$q\to1$} in~the $q$-Gaussian kernel~\eqref{qG_ker}.
The Rational Quadratic kernel is of the form
\begin{equation}
\label{psi2}
\psi_2(x,y) = \left(1 - \frac{\Vert{x-y}\Vert_2^2}{\Vert{x-y}\Vert_2^2 + c} \right), 
\end{equation}
where {$c\in\mathbb{R}$}, {$c>0$}. Substituting {$q=2$} in~\eqref{qG_ker}, we obtain~\eqref{psi2}
with {$c=\sigma^2$}. 
The Laplacian kernel is defined as
\begin{equation}
\label{psi3}
\psi_3(x,y) = \exp\left(- \frac{\Vert{x-y}\Vert_1}{\sigma} \right), 
\end{equation}
where {$\sigma\in\mathbb{R}$}, {$\sigma>0$}. We 
can retrieve~\eqref{psi3} as {$q\to1$} in~the $q$-Laplacian kernel~\eqref{qL_ker}.

\section{Reproducing Kernel Hilbert Space}
\label{rkhs}

As discussed earlier, kernels map the data points to
a higher dimensional feature space, also called the Reproducing Kernel Hilbert Space (RKHS)
that is unique for each positive definite kernel~\cite{Aronszajn_1950_jour_AMSTrans}.
The significance of RKHS
for support vector kernels using Bochner's theorem~\cite{Bochner_1959_book_Princeton}, which provides
a RKHS in Fourier space for translation invariant kernels, is stated in~\cite{Smola_1998_jour_NeuNet}.
Other approaches also exist that lead to explicit description
of the Gaussian kernel~\cite{Steinwart_2006_jour_ITTrans}, but this approach does not work
for the proposed kernels as Taylor series expansion of the $q$-exponential
function~\eqref{q_exp} does not converge for {$q>1$}. So, we follow the Bochner's approach.  

We state Bochner's theorem, and then use the method presented in~\cite{Hofmann_2008_jour_AnnStat} 
to show how it can be used to construct the RKHS for a p.d. kernel.

\begin{theorem}[\textbf{Bochner}]
\label{PD_result7}
A continuous kernel {$\varphi(x, y) = \varphi(x-y)$} on {$\mathbb{R}^d$} is positive definite if
and only if {$\varphi(t)$} is the Fourier transform of a non-negative measure,
i.e., there exists {$\rho\geqslant0$} such that 
{$\rho(\omega)$} is the inverse Fourier transform of {$\varphi(t)$}.
\end{theorem}

Then, the RKHS of the kernel {$\varphi$} is given by
\begin{equation}
\mathcal{H}_\varphi = \left\{ f\in \mathit{L}^2(\mathbb{R}) \bigg\vert \int\limits_{-\infty}^{\infty} 
\left.\frac{\vert\hat{f}(\omega)\vert^2}{\rho(\omega)}\mathrm{d}\omega\right. <\infty\right\}
\end{equation}
with the inner product defined as
\begin{equation}
\langle{f,g}\rangle_\varphi = \int\limits_{-\infty}^{\infty} 
\frac{\hat{f}(\omega)\overline{\hat{g}(\omega)}}{\rho(\omega)}\mathrm{d}\omega,
\end{equation}
where {$\hat{f}(\omega)$} is the Fourier transform of {$f(t)$} and 
{$\mathit{L}^2(\mathbb{R})$} is set of all functions on {$\mathbb{R}$}, square integrable with respect to the Lebesgue measure.

It must be noted here that in our case, the existence and non-negativity of 
the inverse Fourier transform $\rho$ is obvious
due to the positive definiteness of the proposed kernels (Corollaries~\ref{qG_pd} and~\ref{qL_pd}).
Hence, to describe the RKHS it is enough to determine an expression for $\rho$ for both the kernels.
We define the functions
corresponding to the $q$-Gaussian and $q$-Laplacian kernels, respectively, as
\begin{equation}
\varphi_G(t) = \left(1 + \frac{(q-1)}{(3-q)\sigma^2}\sum_{j=1}^{N}t_j^2 \right)^{\frac{1}{1-q}}, 
\quad q\in(1,3), 
\end{equation}
and
\begin{equation}
\varphi_L(t) = \left(1 + \frac{(q-1)}{(2-q)\beta}\sum_{j=1}^{N}|t_j|\right)^{\frac{1}{1-q}}, 
\quad q\in(1,2),
\end{equation}
where {$\beta,\sigma\in\mathbb{R}$}, {$\beta>0$} and {$t = (t_1,\ldots,t_N) \in \mathbb{R}^N$}.
We derive expressions for their inverse Fourier transforms {$\rho_G(\omega)$} and {$\rho_L(\omega)$}, respectively.
For this, we require a technical result~\cite[Eq. 4.638(3)]{Gradshteyn_1994_book_Elesevier},
which is stated in the following lemma.

\begin{lemma}
\label{PD_result8}
Let $s\in(0,\infty)$ and $p_i,q_i,r_i\in(0,\infty)$ for $i=1,2,\ldots,N$
be constants, then the $N$-dimensional integral
\begin{align*}
 \int_0^{\infty}\int_0^{\infty}\ldots\int_0^{\infty}
 &\frac{\prod_{i=1}^{N} x_i^{p_i-1}}{\left(1+\sum_{i=1}^{N} (r_ix_i)^{q_i}\right)^s}
 \mathrm{d}x_1 \mathrm{d}x_2 \ldots \mathrm{d}x_N
 \\=&
 \frac{\Gamma\left(s - \sum_{i=1}^{N} \frac{p_i}{q_i}\right)}{\Gamma(s)}
 \prod_{i=1}^{N} \left(\frac{\Gamma\left(\frac{p_i}{q_i}\right)}{q_ir_i^{p_iq_i}}\right)\;.
\end{align*}
\end{lemma}

We now derive the inverse Fourier transforms. We prove the result for Proposition~\ref{qG_fourier}.
The proof of Proposition~\ref{qL_fourier} proceeds similarly.

\begin{proposition}
\label{qG_fourier}
The inverse Fourier transform for {$\varphi_G(t)$} is given by 
\begin{align}
\rho_G&(\omega) = 
\frac{1}{\left(\frac{\sqrt{2}(q-1)}{(3-q)\sigma^2}\right)^N \Gamma\left(\frac{1}{q-1}\right)} 
\times \nonumber \\&
\sum_{b=0}^{\infty} \frac{(-1)^b}{b!} \Gamma\left(\frac{1}{q-1}-\frac{N}{2}-b\right)
\left(\frac{(3-q)\sigma^2\Vert{\omega}\Vert_2}{2(q-1)}\right)^{2b}.
\end{align}
\end{proposition}

\begin{proof}
 By definition,
\begin{equation}
 \rho_G(\omega) = (2\pi)^{-N/2}
 \int_{\mathbb{R}^N} \exp(\mathrm{i} t\cdot\omega) \varphi_G(t) \mathrm{d}t\;.
\end{equation}
Expanding the exponential term, we have
\begin{align*}
 \exp(\mathrm{i}t\cdot\omega) 
 = \prod_{j=1}^{N} \big( \cos(t_j\omega_j) + \mathrm{i}\sin(t_j\omega_j)\big)\;.
\end{align*}
Since, both {$\cos(t_j\omega_j)$} are {$\varphi_G(t)$} are even functions for every $t_j$,
while {$\sin(t_j\omega_j)$} is an odd function, hence integrating over {$\mathbb{R}^{N}$},
all terms with a sin component become zero. Further, the remaining term is odd, and hence,
the integral is same in every orthant. So the expression reduces to
\begin{align}
 &\rho_G(\omega) = \nonumber \\
 &\left(\frac{2}{\pi}\right)^{\frac{N}{2}}
 \int\limits_{0}^{\infty} .. \int\limits_{0}^{\infty}
 \bigg(1 + c\sum_{j=1}^{N}t_j^2 \bigg)^{\frac{1}{1-q}} \prod_{j=1}^{N} \cos(t_j\omega_j) \mathrm{d}t_1 ... \mathrm{d}t_N
\label{rhoG_step1}
\end{align}
where {$c=\frac{(q-1)}{(3-q)\sigma^2}$}. Each of the cosine term can be expanded in form of an infinite series as
\begin{displaymath}
 \cos(t_j\omega_j) = \sum_{m_j=0}^{\infty} (-1)^{m_j} \frac{\omega_j^{2m_j} t_j^{2m_j}}{(2m_j)!}\;.
\end{displaymath}
Substituting in~\eqref{rhoG_step1} and using Lemma~\ref{PD_result8}, we obtain
\begin{align}
 &\rho_G(\omega) = \frac{1}{\left(\sqrt{2\pi}c\right)^{N} \Gamma\left(\frac{1}{q-1}\right)} \times\nonumber \\
 & \sum_{m_1,..,m_N = 0}^{\infty} \left(-\frac{1}{c^2}\right)^{\sum\limits_{j=1}^{N} m_j} \Gamma\bigg(\frac{1}{q-1} -
 \frac{N}{2} - \sum_{j=1}^{N} m_j \bigg)
 g(\omega)
\label{rhoG_step2}
\end{align}
where 
\begin{equation}
 g(\omega) = \prod_{j=1}^{N} \frac{\omega_j^{2m_j} \Gamma\left(m_j + \frac{1}{2}\right)}{(2m_j)!}\;.
 \label{rhoG_step3}
\end{equation}
Using expansion of gamma function for half integers, we can write~\eqref{rhoG_step3} as
\begin{equation}
 g(\omega) = \pi^{N/2}\prod_{j=1}^{N} \frac{\omega_j^{2m_j}}{4^{m_j} m_j!}\;.
\end{equation}
Substituting in~\eqref{rhoG_step2} and using {$b=\sum_{j=1}^{N} m_j$}, we have
\begin{align}
 &\rho_G(\omega) = \frac{1}{\left(\sqrt{2}c\right)^{N} \Gamma\left(\frac{1}{q-1}\right)} \times\nonumber \\
 & \sum_{b = 0}^{\infty} \left(\frac{-1}{4c^2}\right)^{b} \Gamma\left(\frac{1}{q-1} - \frac{N}{2} - b\right)
 \sum_{{m_1,..,m_N}\atop{\sum\limits_{k=1}^{N} m_k =b}}
 \frac{\omega_1^{2m_1} ... \omega_N^{2m_N}}{m_1! ... m_N!}
 \label{rhoG_step4}
\end{align}
We arrive at the claim by observing that the terms in the inner summation in~\eqref{rhoG_step4} are similar to 
terms of multinomial expansion of {$\frac{1}{b!} \left(\omega_1^2 + \ldots + \omega_N^2\right)^b$}. 
\end{proof}

It can be observed that the above result agrees with the fact that inverse Fourier
transform of radial functions are radial in nature.
We present corresponding result for $q$-Laplacian kernel. 

\begin{proposition}
\label{qL_fourier}
The inverse Fourier transform for {$\varphi_L(t)$} is given by 
\begin{align}
\rho_L &(\omega) = 
\frac{1}{\left(\frac{\sqrt{\pi}(q-1)}{(2-q)\beta\sqrt{2}}\right)^N \Gamma\left(\frac{1}{q-1}\right)} 
\times \nonumber \\&
\sum_{b=0}^{\infty} (-1)^b\Gamma\left(\frac{1}{q-1}-N-2b\right)
\left(\frac{(2-q)\beta}{(q-1)}\right)^{2b} g_b(\omega),
\end{align}
where 
\begin{equation}
g_b(\omega) = \sum_{{m_1,\ldots,m_N\in\mathbb{N}}\atop{\sum\limits_{k=1}^{N} m_k =b}}
\omega_1^{2m_1} \omega_2^{2m_2} \ldots \omega_N^{2m_N} 
\end{equation}
with
{$\omega_1,\ldots,\omega_N$} being the components of {$\omega$}.
\end{proposition}

\section{Performance Comparison}
\label{result}

In this section, we apply the $q$-Gaussian and $q$-Laplacian kernels in 
classification and regression. 
We provide insights into the behavior
of these kernels through examples. We also compare the 
performance of the kernels for different values of $q$, and also with the Gaussian, Laplacian (\textit{i.e.}, when $q\to1$),
and polynomial kernels 
using various data sets from UCI repository~\cite{Frank_2010_misc_UCI}. 
The simulations have been performed using LIBSVM~\cite{Chang_2011_jour_ACMTIST}.
Table~\ref{dataset} provides a description of the data sets used. 
The last few data sets have been used for regression.
\begin{table}[h]
\small
\centering
\begin{tabular}{|c|c|c|c|c|}
\hline
\multicolumn{2}{|c|}{Data Set}	& Class	& Attribute	& Instance	\\
\hline\hline
1	& Acute Inflammations	& 2		& 6		& 120		\\	
2	& Australian Credit$^*$	& 2		& 14		& 690		\\
3	& Blood Transfusion	& 2		& 4		& 748		\\
4	& Breast Cancer$^*$	& 2		& 9		& 699		\\
5	& Iris			& 3		& 4		& 150		\\
6	& Mammographic Mass	& 2		& 5		& 830		\\
7	& Statlog (Heart)$^*$	& 2		& 13		& 270		\\
8	& Tic-Tac-Toe		& 2		& 9		& 958		\\
9	& Vertebral Column	& 3		& 6		& 310		\\
10	& Wine$^*$		& 3		& 13		& 178		\\
\hline \hline
11	& Auto MPG		& --		& 8 		& 398		\\
12	& Servo			& --		& 4		& 167		\\
13	& Wine Quality (red)	& --		& 12		& 1599		\\
\hline
\end{tabular}
\caption{Data Sets {(sets marked $^*$ have been normalized)}. }
\label{dataset}
\end{table}

\subsection{Kernel SVM}

Support Vector Machines (SVMs) are one of the most 
important class of kernel machines. While linear SVMs, using inner product
as similarity measure, are quite common, other variants using various kernel
functions, mostly Gaussian, are also used in practice. Use of kernels leads to
non-linear separating hyperplanes, which sometimes provide better classification.
Now, we formulate a SVM based on the proposed kernels. For the $q$-Gaussian kernel~\eqref{qG_ker}, it leads to an 
optimization problem with the following dual form:
\begin{align*}
\min_{\alpha\in\mathbb{R}^n} &\sum_{i=1}^{n} \alpha_i - 
\frac{1}{2}\sum_{i,j=1}^{n} \alpha_i\alpha_j y_i y_j 
\exp_q \left(-\frac{\Vert{x_i-x_j}\Vert_2^2}{(3-q)\sigma^2}\right)
\label{svm_formula}
\end{align*}
subject to 
{$\alpha_i \geqslant0$}, {$i = 1,\ldots,n$}, and
{$\sum_{i=1}^{n} \alpha_i y_i = 0$},
where, {$\{x_1,\ldots,x_n\}\subset\mathcal{X}$} are the training data points
and {$\{y_1,\ldots,y_n\}\subset\{-1,1\}$} are the true classes. The optimization
problem for the $q$-Laplacian kernel~\eqref{qL_ker} can be formulated by using
{$\exp_q \left(-\frac{\Vert{x_i-x_j}\Vert_1}{(2-q)\beta}\right)$} in the above expression.

\begin{figure}[ht]
 \centering
\includegraphics[width=0.42\textwidth]{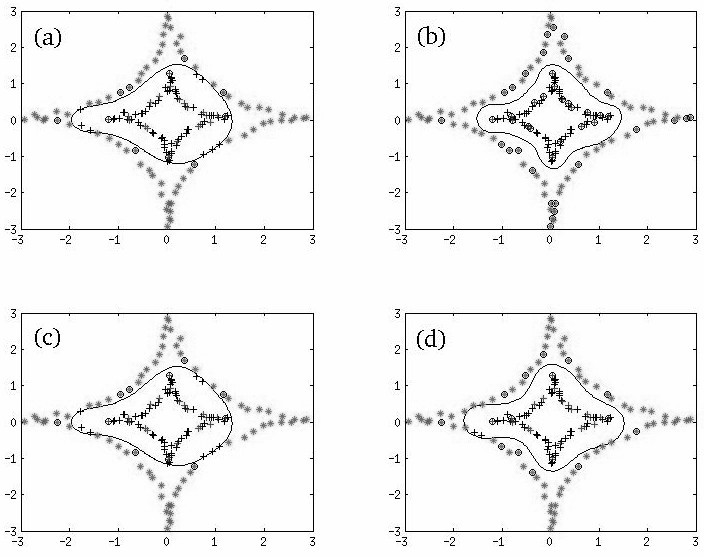}
\caption{Decision boundaries using (a) Gaussian, (b) $q$-Gaussian $(q=2.95)$, (c) Laplacian,
and (d) $q$-Laplacian $(q=1.95)$ kernel SVMs.}
 \label{svm_example}
\end{figure}

\begin{table*}[ht]
\small
\centering
\begin{tabular}{|c|c|c|c|c|c|c|c|c|c|c|c|c|}
\hline
\multicolumn{2}{|c|}{Data Sets}	& 1 & 2 & 3 & 4 & 5 & 6 & 7 & 8 & 9 & 10 \\
\hline
\multicolumn{2}{|c|}{Parameter ($\sigma = \sqrt{\beta}$)}	
&	10	&	15	&	5	&	5	&	2	&	25	&	5	&	1.5	&	50	&	1	\\
\hline \hline
\multicolumn{2}{|c|}{Gaussian $(q\to1)$}		
			&	86.67	&	76.96	&	\textbf{77.27}	&	96.63	&	97.33	&	79.28	&	82.96	&	\textbf{89.46}	&	\textbf{87.10}	&	97.19	\\
\hline 
\multirow{8}{*}{\begin{sideways} {\scriptsize $q$-Gaussian} \end{sideways}}		
&	$q = 1.25$	&	86.67	&	82.46	&	77.14	&	96.63	&	97.33	&	79.40	&	82.96	&	89.25	&	\textbf{87.10}	&	97.19	\\
&	$q = 1.50$	&	86.67	&	83.19	&	76.87	&	96.93	&	97.33	&	79.52	&	83.33	&	89.25	&	\textbf{87.10}	&	97.75	\\
&	$q = 1.75$	&	88.33	&	\textbf{86.38}	&	76.60	&	96.93	&	\textbf{98.00}	&	79.76	&	83.33	&	88.94	&	86.45	&	97.75	\\
&	$q = 2.00$	&	88.33	&	85.80	&	76.74	&	96.93	&	\textbf{98.00}	&	79.88	&	83.33	&	88.62	&	86.13	&	97.75	\\
&	$q = 2.25$	&	89.17	&	85.51	&	76.60	&	96.93	&	\textbf{98.00}	&	79.40	&	83.33	&	87.68	&	85.48	&	\textbf{98.31}	\\
&	$q = 2.50$	&	91.67	&	85.51	&	76.34	&	96.93	&	97.33	&	80.00	&	\textbf{84.07}	&	85.49	&	85.48	&	\textbf{98.31}	\\
&	$q = 2.75$	&	98.33	&	85.51	&	76.47	&	\textbf{97.22}	&	96.67	&	\textbf{80.48}	&	\textbf{84.07}	&	84.34	&	85.16	&	\textbf{98.31}	\\
&	$q = 2.95$	&	\textbf{100}	&	85.51	&	75.53	&	\textbf{97.22}	&	96.67	&	80.12	&	82.22	&	75.99	&	85.16	&	97.75	\\
\hline
\hline
\multicolumn{2}{|c|}{Laplacian $(q\to1)$}		
			&	93.33	&	\textbf{86.23}	&	77.81	&	97.07	&	\textbf{96.67}	&	81.69	&	\textbf{83.70}	&	\textbf{94.89}	&	76.45	&	\textbf{98.88}	\\
\hline 
\multirow{4}{*}{\begin{sideways} {\scriptsize $q$-Laplacian} \end{sideways}}		
&	$q = 1.25$	&	95.83	&	85.51	&	\textbf{77.94}	&	97.07	&	\textbf{96.67}	&	81.57	&	\textbf{83.70}	&	92.80	&	77.42	&	\textbf{98.88}	\\
&	$q = 1.50$	&	97.50	&	85.51	&	77.27	&	97.07	&	\textbf{96.67}	&	81.81	&	\textbf{83.70}	&	89.67	&	77.10	&	\textbf{98.88}	\\
&	$q = 1.75$	&	\textbf{100}	&	85.51	&	77.14	&	97.51	&	\textbf{96.67}	&	82.29	&	83.33	&	84.55	&	78.39	&	\textbf{98.88}	\\
&	$q = 1.95$	&	\textbf{100}	&	85.51	&	75.67	&	\textbf{97.80}	&	96.00	&	\textbf{83.73}	&	82.96	&	71.09	&	\textbf{86.77}	&	95.51	\\
\hline
\hline
\multirow{4}{*}{\begin{sideways} {\scriptsize Polynomial} \end{sideways}}		
&	$d = 1$ (linear)&	\textbf{100}	&	85.51	&	72.86	&	97.07	&	\textbf{98.00}	&	82.17	&	\textbf{83.70}	&	65.34	&	85.16	&	97.19	\\
&	$d = 2$		&	\textbf{100}	&	85.22	&	76.47	&	96.19	&	96.67	&	\textbf{83.86}	&	80.37	&	86.53	&	76.77	&	96.63	\\
&	$d = 5$		&	\textbf{100}	&	80.72	&	76.47	&	95.61	&	95.33	&	83.61	&	74.81	&	\textbf{94.15}	&	64.84	&	94.94	\\
&	$d = 10$	&	\textbf{100}	&	76.23	&	76.47	&	94.00	&	94.67	&	81.69	&	74.81	&	88.73	&	59.03	&	93.26	\\
\hline
\end{tabular}
\caption{Percentage of correct classification in kernel SVM using 5-fold cross validation.}
\label{svm_table}
\end{table*}

The two-dimensional example in Figure~\ref{svm_example} illustrates the nature of hyperplanes that can be 
obtained using various kernels. The decision boundaries are more flexible for
$q$-Laplacian and $q$-Gaussian kernels. 
Further, viewing the 
Laplacian and Gaussian kernels as special cases ($q\to1$), it can be said that increase in the
value of $q$ leads to more flexibility of the decision boundaries.

We compare the performance of the proposed kernels with Gaussian and Laplacian kernel SVMs
for various values of $q$. The results of 5-fold cross validation 
using multiclass SVM are shown in Table~\ref{svm_table}.
Further, the power-law nature reminds
practitioners of the popular polynomial kernels
\begin{displaymath}
 P_d(x,y) = \left(x^{T}y+c\right)^d, \qquad \text{for } x,y\in\mathbb{R}^{N},
\end{displaymath}
where the parameters $c\in(0,\infty)$ and $d\in\mathbb{N}$.
Hence, we also provide the accuracies obtained using these kernels.
We have fixed particular {$\sigma$} for each data set,
and consider $\beta$ is fixed at $\beta=\sigma^2$. 
For polynomial kernels, we consider $c=0$, while $d$ is varied. 
The best values of $q$ among all Gaussian type and Laplacian type kernels are marked for each data set. 
In case of the polynomial kernels,
we only mark those cases where the best results among these kernels is better
or comparable with the best cases of Gaussian or Laplacian types.
We note here that in the simulations, the polynomial kernels
required normalization of few other data sets as well.

The results indicate the significance of tuning the parameter $q$. 
For most cases, the $q$-Gaussian and $q$-Laplacian kernels tend to perform better than their exponential counterparts.
This can be justified by the flexibility of the separating hyperplane 
achieved. However, it has been observed (not demonstrated here)
that for very high or very low values of {$\sigma$} (or $\beta$), the
kernels give similar results, which happens because the power-law 
and the exponential natures cannot be distinguished in such cases.
The polynomial kernels, though sometimes comparable to the proposed kernels,
rarely improves upon the performance of the power-law kernels.

\subsection{Kernel Regression}

In linear basis function models for regression, 
given a set of data points, the output function is 
approximated as a linear combination of fixed non-linear functions as
{$f(x) = w_0 + \sum_{j=1}^{M} w_j \phi_j(x)$},
where {$\{\phi_1(.),\ldots,\phi_M(.)\}$} are the basis functions,
usually chosen as 
{$\phi_j(x) = \psi(x,x_j)$},
{$x_1,\ldots, x_M$} being the given data points, and $\psi$ a p.d. kernel. The constants
{$\{w_0, w_1, \ldots, w_M\}$} are obtained by minimizing least squared error. 
Such an optimization can be formulated as an
$\epsilon$-Support Vector type problem~\cite{Smola_2004_jour_StatComp}.
The kernels defined in~\eqref{qG_ker} and~\eqref{qL_ker} can also be used in this case as shown in the example in Figure~\ref{reg_example},
where $\epsilon$-SV regression is used to reconstruct a sine wave from 20 uniformly spaced sampled data points in $[0,\pi]$.
\begin{figure}[h]
\centering
\includegraphics[width=0.35\textwidth]{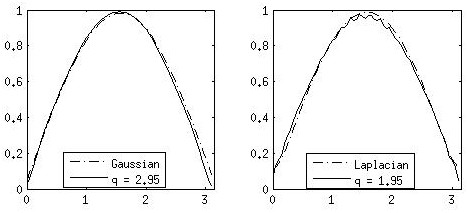}
\caption{Sine curve obtained by $\epsilon$-SVR using Gaussian, $q$-Gaussian $(q=2.95)$, 
Laplacian and $q$-Laplacian $(q=1.95)$ kernels with $\sigma = \sqrt{\beta} = 2$ and $\epsilon = 0.01$.}
\label{reg_example}
\end{figure}

\begin{table}[ht]
\small
\centering
\begin{tabular}{|c|c|c|c|c|}
\hline
\multicolumn{2}{|c|}{Data Sets}	& 11 & 12 & 13 \\
\hline
\multicolumn{2}{|c|}{Parameter ($\sigma = \sqrt{\beta}$)}
&	1	&	2	&	10	\\
\hline \hline
\multicolumn{2}{|c|}{Gaussian $(q\to1)$}		
&	11.1630	&	0.9655	&	0.4916	\\
\hline 
\multirow{8}{*}{\begin{sideways} {\scriptsize $q$-Gaussian} \end{sideways}}		
&	$q = 1.25$	&	11.0694	&	0.9218	&	0.4883	\\
&	$q = 1.50$	&	10.9674	&	0.9035	&	0.4853	\\
&	$q = 1.75$	&	10.8826	&	\textbf{0.8986}	&	0.4823	\\
&	$q = 2.00$	&	10.7406	&	0.9005	&	0.4781	\\
&	$q = 2.25$	&	10.5661	&	0.9072	&	0.4734	\\
&	$q = 2.50$	&	\textbf{10.4428}	&	0.9424	&	0.4661	\\
&	$q = 2.75$	&	10.4796	&	1.0698	&	0.4595	\\
&	$q = 2.95$	&	12.2427	&	1.5439	&	\textbf{0.4419}	\\
\hline
\hline
\multicolumn{2}{|c|}{Laplacian $(q\to1)$}		
&	\textbf{9.7681}	&	\textbf{0.5398}	&	0.4298	\\
\hline 
\multirow{4}{*}{\begin{sideways} {\scriptsize $q$-Laplacian} \end{sideways}}		
&	$q = 1.25$	&	10.2052	&	0.5532	&	0.4223	\\
&	$q = 1.50$	&	10.9578	&	0.6055	&	0.4123	\\
&	$q = 1.75$	&	13.2213	&	0.7910	&	0.3961	\\
&	$q = 1.95$	&	17.7303	&	1.6934	&	\textbf{0.3784}	\\
\hline
\hline
\multirow{4}{*}{\begin{sideways} {\scriptsize Polynomial} \end{sideways}}		
&	$d = 1$ (linear)&	13.3765	&	1.9047	&	0.4357	\\
&	$d = 2$		&	10.5835	&	2.2740	&	\textbf{0.4268}	\\
&	$d = 5$		&	16.8173	&	2.3305	&	0.5485	\\
&	$d = 10$	&	52.4609	&	2.7358	&	10.5518	\\
\hline
\end{tabular}
\caption{Mean Squared Error in kernel SVR.}
\label{reg_table}
\end{table}

The performance of the proposed kernels have been compared with polynomial, Gaussian and Laplacian kernels
for various values of $q$ using data sets 11, 12 and 13. The results of 5-fold cross validation 
using $\epsilon$-SVR $(\epsilon = 0.1)$ are shown in Table~\ref{reg_table}.
We fixed particular {$\beta=\sigma^2$} for each data set. 
Though Laplacian kernel seems to outperform its power-law variants, the 
$q$-Gaussians dominate the performance of Gaussian kernel. The results further indicate that the error
is a relatively smooth function of $q$, and does not have a fluctuating behavior, though its trend
seems to depend on the data.
The relative performance of the polynomial kernels is poor. 

\section{Conclusion}
\label{conclusion}

In this paper, we proposed a power-law generalization of Gaussian and Laplacian 
kernels based on Tsallis distributions. They retain their properties in 
the classical case as $q\to1$.
Further, due to their power-law nature, the tails of the proposed kernels decay at a 
slower rate than their exponential counterparts, which in turn 
broadens the use of these kernels in learning tasks.

We showed that the proposed kernels are positive definite for certain range of {$q$},
and presented results pertaining to the RKHS of the proposed kernels
using Bochner's theorem.
We also demonstrated the performance of the proposed kernels in support vector classification and regression. 

The power-law behavior was recognized long time back in 
many problems in the context of statistical analysis.
Recently power-law distributions have been studied in machine learning communities.
As far as our knowledge, this is the first paper that introduces and studies power-law kernels, 
leading to the notion of a \emph{``fat-tailed kernel machine''}.

\nocite{Turan_1976_book_AkademiaKiado}

\end{document}